\documentclass[twoside]{article}

\usepackage{algorithm}
\usepackage{algpseudocode}
\usepackage{amsmath}
\usepackage{amsthm}
\usepackage{commath}
\usepackage{amssymb}
\usepackage{graphicx}
\usepackage{hyperref}
\usepackage{comment}

\newtheorem{theorem}{Theorem}
\newtheorem*{theorem*}{Theorem}

\newtheorem{lemma}{Lemma}[theorem]
\newtheorem{definition}{Definition}

\usepackage[accepted]{aistats2020}
\usepackage[round]{natbib}

\setcitestyle{authoryear, open={(}, close={)}}
\allowdisplaybreaks

\begin{document}

%

%

\twocolumn[
\aistatstitle{Multiplicative Gaussian Particle Filter}
\aistatsauthor{ Xuan Su \And Wee Sun Lee \And  Zhen Zhang }
\aistatsaddress{ National University of Singapore \And  National University of Singapore  \And University of Adelaide  }
]

\begin{abstract}
We propose a new sampling-based approach for approximate inference in filtering problems. Instead of approximating conditional distributions with a finite set of states, as done in particle filters, our approach approximates the distribution with a weighted sum of functions from a set of continuous functions. Central to the approach is the use of sampling to approximate multiplications in the Bayes filter. We provide theoretical analysis, giving conditions for sampling to give good approximation. We next specialize to the case of weighted sums of Gaussians, and show how properties of Gaussians enable closed-form transition and efficient multiplication. Lastly, we conduct preliminary experiments on a robot localization problem and compare performance with the particle filter, to demonstrate the potential of the proposed method.
\end{abstract}

\section{Introduction}

Sequential state estimation is a general class of problems arising frequently in areas such as  computer vision and robotics. Given a sequence of potentially noisy, incomplete data about a dynamic system, the task is to establish the value of a quantity of interest about that system. A number of filtering methods built on top of the Bayes filter such as the particle filter have been successfully applied to tackle such problems.

In the particle filter, a set of weighted states (particles) from the state space is used to approximate the conditional state distribution, also known as the \emph{belief}. At each time step, each particle is propagated according to the system dynamics and then compared against current observation inputs to determine a new weight. This is followed by an optional resampling step that can remove relatively unimportant particles and focus on important parts of the belief. For a continuous state space, a large set of particles may be required in order to give a good approximation of a belief distribution. 

In this paper, we seek to improve the performance of particle filters in continuous space by approximating the belief with a weighted sum of continuous functions, instead of a weighted set of particles. The Bayes filter requires two fundamental operations, marginalization and multiplication. We consider continuous function classes that are closed under marginalization and multiplication, and further assume that these operations are computationally efficient. Examples of function classes with these properties include monomials and Gaussians. While these functions classes, which we call the base function classes, have limited representation power, weighted sums of functions from these classes are universal approximators. 

We assume that the transition and observation functions of the filter are represented using finite weighted sums of the base functions. With the use of finite weighted sums, marginalization continues to be efficiently computable. We show how multiplications of two weighted sums could be approximated efficiently using sampling, and give the convergence rate of approximating the Bayes filter when the base function class has a finite complexity measure, called the co-VC-dimension.  We then focus on the special case of Gaussian functions as the base functions. We show that the class of Gaussian functions has linear co-VC-dimension. In addition, Gaussians allow a closed-form transition update and efficient multiplication, making it appropriate for our proposed technique. Finally, we test our new inference method on a robot localization problem to demonstrate its potential, comparing its performance to that of the particle filter.

\section{Background}
In a filtering problem, we have a sequence of observations $\mathbf{z}_1,\ldots,\mathbf{z}_t$ and would like to estimate the state $\mathbf{x}_t$. We assume that the observation $\mathbf{z}_t$ depends only on the state $\mathbf{x}_t$ through the observation function $p(\mathbf{z}_t\mid\mathbf{x}_t)$ and that the system is Markovian with transition function $p(\mathbf{x}_t \mid \mathbf{x}_{t-1}, \mathbf{u}_{t-1} \ldots, \mathbf{x}_1, \mathbf{u}_1) = p(\mathbf{x}_t \mid \mathbf{x}_{t-1}, \mathbf{u}_{t-1})$ where $\mathbf{u}_t$ is the control input at time $t$.

Solutions proposed to filtering problems are often based on the classic \textit{Bayes filter} (see e.g. \citet{thrun2005probabilistic}), described in Algorithm \ref{algo:bayes}. The Bayes filter is recursive, where the belief $bel_t(\mathbf{x})$ at time $t$ is computed from the belief $bel_{t-1}(\mathbf{x})$ at time $t - 1$. It operates on three main components -- belief $bel(\mathbf{x})$, transition function $p(\mathbf{x} \mid \mathbf{u}, \mathbf{x'})$ and observation function $p(\mathbf{z} \mid \mathbf{x})$, and consists of two essential steps:  prediction (line 3), and correction (line 4). The prediction step calculates probabilities of states using prior belief and control data; whereas the correction step considers the actual observations to re-normalize the belief.

\begin{algorithm}
  \caption{The Bayes Filter}
  \begin{algorithmic}[1]
    \Procedure{BayesFilter}{$bel_{t-1}(\mathbf{x}), \mathbf{u}, \mathbf{z}$}
      \For{all \texttt{$\mathbf{x}$}}
        \State $\bar{bel}_{t}(\mathbf{x}) \gets \int p(\mathbf{x} \mid \mathbf{u}, \mathbf{x'})  \hspace{0.05cm} bel_{t-1}(\mathbf{x'}) d\mathbf{x'}$
        \State $bel_{t}(\mathbf{x}) \gets \eta \hspace{0.05cm} p(\mathbf{z} \mid \mathbf{x}) \hspace{0.05cm} \bar{bel}_{t}(\mathbf{x})$
      \EndFor
      \State \textbf{return} $bel_{t}(\mathbf{x})$
    \EndProcedure
  \end{algorithmic}
  \label{algo:bayes}
\end{algorithm}

\subsection{Related Works}
A variety of implementations of the Bayes filter have been proposed in the past decades. The earliest work -- the celebrated Kalman filter (KF) \citep{kalman1960new} uses a single multivariate normal to represent the belief $bel_{t}(\mathbf{x})$. Both transition and observation are assumed to be linear, which guarantees that the algorithm always operates on Gaussians. However, assumptions in the KF are too restrictive as most dynamic systems in practice are non-linear. The extended Kalman filter (EKF) due to \cite{julier1997new} is an enhancement of the KF that adapts the KF to account for non-linearity in real-world systems, under the same single-Gaussian assumption.

Many densities cannot be accurately described using one Gaussian only, making Kalman-type filters invalid in many cases. To improve the representation power, mixture of Gaussians is often used. An early work, the Gaussian sum filter (GSF) \citep{alspach1972nonlinear}, approximates beliefs using a Gaussian Mixture Model (GMM) and includes an EKF estimator for each GMM component, endowing it with the ability to model non-Gaussian densities.

In contrast to aforementioned filters which adopt a particular functional form, \emph{nonparameteric filters} make no such assumptions. Instead, the posteriors are approximated using a set of values sampled from the state space.
The histogram filter decomposes the state space into a finite number of regions or ``histograms'', and represents each using a probability value.
Another prominent example of nonparametric filters, the {particle filter} (PF) \citep{gordon1993novel}, relies on a weighted set of states drawn from $bel_t(\mathbf{x})$ for its representation. Propagated through time, the particles often encounter a depletion problem where some particles are assigned degenerate weights. Resampling is used to deal with such problems.

Many recent attempts incorporate both Gaussians and particles to parameterize the Bayes filter, often using alternative resampling methods. For instance, \cite{faubel2009split} develops a procedure based on the GSF, that modifies the re-appoximation step to ``split'' Gaussians in important regions of distribution and ``merge'' them in unlikely ones, controling its granularity while representing beliefs. \cite{psiaki2015gaussian} invents another resampling algorithm for Gaussian sums. Central to the approach are the ideas of upper-bounding Gaussian covariance values rendering them tenable, and controlling the number of terms in the sum using ad-hoc strategies.  Another variant of resampling is proposed in particle Gaussian mixture filter \citep{raihan2016particle}, which performs GMM clustering at each step to include less significant particles, instead of doing resampling to remove them.

Our work follows the same spirit of using both Gaussians and particles. However, we represent the transition and observation functions with weighted sums of Gaussians, and use sampling to approximate multiplications, yielding a new form of particle resampling. Our work builds on top of the work in  \cite{wrigley2017tensor}. Their work focuses on discrete problems, using sampling to do multiplications for approximate inference over a junction tree. We focus on the continuous case for filtering problems.

\section{Theoretical Analysis: General Case}
\label{sec:theory_general}
In this section, we seek to answer the question: under what conditions does our use of sampling to approximate multiplications work well for filtering problems? To this end, formalizing our inference method in mathematical language is necessary.

Let $\mathcal{F}$ be a class of real-valued functions over input space $\mathcal{X}$, closed under multiplication and marginalization. We approximate functions including the belief $bel_{t}(\mathbf{x})$, the transition $p(\mathbf{x} \mid \mathbf{u}, \mathbf{x'})$ and the observation $p(\mathbf{z} \mid \mathbf{x})$, all as weighted sums of functions $f \in \mathcal{F}$. As $\mathbf{u}$ and $\mathbf{z}$ represent control and observation inputs, the functions are over $\mathbf{x}$ and prior state $\mathbf{x'}$: $f(\mathbf{x}, \mathbf{x'}) = \sum_{i=1}^{K} w_i f_i(\mathbf{x}, \mathbf{x'})$.

In the Bayes filter, the two main operations are in lines 3 and 4. Integration is required in line 3 while multiplications of belief with transition and observation are required in lines 3 and 4 respectively. We assume that the integration can be done efficiently. Each multiplication with a weighted sum of $K$ functions gives an extra $K$ factor to the number of functions in the summation. Over $t$ steps, repeated multiplication results in an exponential number of functions. To bring the number of functions under control, we perform sampling after each multiplication, following the method proposed by \citet{wrigley2017tensor}. The key observation is that, after normalization so that the weights sum to one, a weighted sum of functions can be regarded as a probability distribution over the functions with expectation equal to the true value, by considering the weight of each term as its probability.

To multiply two weighted sums of functions, $\sum_{i=1}^{K_1} w_{i} f_{i}(\mathbf{x})$ and $\sum_{j=1}^{K_2} v_{j} h_{j}(\mathbf{x})$, we first perform a full multiplication to obtain the product containing $K_1 K_2$ functions: $\sum_{i=1}^{K_1} \sum_{j=1}^{K_2} w_{i} v_{j} f_{i} h_{j}$. As the terms $w_i v_j$ may not sum to one for arbitrary weighted sums, we normalize them for the purpose of probability sampling. Next, a sample of $K$ indices $\{ k_r, l_r \}_{r=1}^{K}$ is drawn from the distribution $w_i v_j$. This gives a new weighted sum of functions, $\sum_{r=1}^{K} \frac{1}{K} f_{k_{r}} h_{l_{r}}$, that is an unbiased estimate of the original multiplication.

We give the convergence rate of this sampling approximation when the update $bel_{t}(\mathbf{x}) = p(\mathbf{z} \mid \mathbf{x}) \int p(\mathbf{x} \mid \mathbf{u}, \mathbf{x'})  bel_{t-1}(\mathbf{x'}) d\mathbf{x'}$ is done $T$ times. The proof uses techniques from \citet{wrigley2017tensor}. We extend the techniques to handle continuous functions through the use of co-VC-dimension, whereas \citet{wrigley2017tensor} only analyzed discrete functions. We first introduce technical definitions.

\begin{definition}[Subgraph] The \textit{subgraph} of a function class $\mathcal{F}$ is defined as the class of sets of the form $\{ (\mathbf{x}, y) \in \mathcal{X} \times \mathbb{R}; y \leq f(\mathbf{x}) \}$ for $f \in \mathcal{F}$. 
\end{definition}
\begin{definition}[Dual]
The \textit{dual} $\mathcal{F'}$ of $\mathcal{F}$ is defined as the class $\{ ev_{\mathbf{x}}; \mathbf{x} \in \mathcal{X} \}$ of evaluation functions, where $ev_{\mathbf{x}}$ satisfies $ev_{\mathbf{x}}(f) = f(\mathbf{x})$ for all $f \in \mathcal{F}$. In words, an evaluation function for a fixed $\mathbf{x}$ takes a function $f \in \mathcal{F}$ as input, and outputs the value of $f(\mathbf{x})$.
\end{definition}
A number of concepts describing complexities of function classes are necessary for our theoretical analysis. These include \emph{VC-dimension} which describes the complexity of binary-valued functions through the notion of \emph{shattering}; \emph{pseudo-dimension} and \emph{co-VC-dimension} which quantify complexities of real-valued functions and the dual of a given function class, respectively.
\begin{definition}[Shattering] Let $\mathcal{G}$ be a class of indicator ($\{0, 1\}$-valued) functions over $\mathcal{X}$. We say $\mathcal{G}$ shatters a set $A \subseteq \mathcal{X}$ if for every subset $E \subseteq A$, there exists some function $g \in \mathcal{G}$ satisfying: 1) $g(\mathbf{x}) = 0$ for every $\mathbf{x} \in A \setminus E$; 2) $g(\mathbf{x}) = 1$ for every $x \in E$.
\end{definition}
\begin{definition}[{VC-dimension}] The VC-dimension of a class of indicator functions $\mathcal{G}$ is the cardinality of the largest set $\mathcal{S} \subseteq \mathcal{X}$ that is shattered by $\mathcal{G}$.
\end{definition}
\begin{definition}[{Pseudo-dimension}]
The pseudo-dimension of a class of real-valued functions $\mathcal{F}$ is the VC-dimension of the subgraph of $\mathcal{F}$. 
\end{definition}
\begin{definition}[{co-VC-dimension}]
The co-VC-dimension of $\mathcal{F}$ is the pseudo-dimension of the dual $\mathcal{F'}$ of $\mathcal{F}$.
\end{definition}
We next present Theorem \ref{thm:1} to aid in the convergence rate derivation in Theorem \ref{thm:2}. Proof of Theorem \ref{thm:1} is given in supplementary materials.
\begin{theorem}
\label{thm:1}
Let $\mathcal{F}$ be a class of real-valued, continuous functions over a set $\mathcal{X}$, with a finite co-VC-dimension $D$. Let $g(\mathbf{x})$ be a function in the convex hull of $\mathcal{F}$: $g(\mathbf{x}) = \sum_{i=1}^{N} w_i f_{i}(\mathbf{x})$, with $\sum_{i=1}^{N} w_i = 1$ and $f_i \in \mathcal{F}$. Assume that functions $f_{i}(\mathbf{x})$ are upper-bounded by $M$ and that the quantity $\int f_{i}(\mathbf{x}) \hspace{0.05cm} d\mathbf{x}$ is lower-bounded by $B$ for all $f_{i}$. Let $P$ be the probability measure over functions $\{ f_1, \dots, f_{N} \}$ such that $P(f_{i}) = w_i$. A sampling operation is taken to draw $K$ functions $\{ h_1, \dots, h_{K} \}$ independently from $P$. Then, for any $\mathbf{x} \in \mathcal{X}$,
\begin{align}
\begin{split}
\label{eq:bound_all_x}
    P\left\{ \frac{1}{K} \sum_{i=1}^{K} h_{i}(\mathbf{x}) \not\in \left[ (1-\zeta) g(\mathbf{x}), (1+\zeta) g(\mathbf{x}) \right] \right\} \\ 
    < 8 (2K)^D \exp{\left( -\frac{\zeta^2}{4} \frac{B}{M} K \right) }
\end{split}
\end{align}
\end{theorem}
Theorem \ref{thm:1} provides a tool to bound values estimated for all $\mathbf{x} \in \mathcal{X}$ using $K$ sampled continuous functions. The theorem comes in handy when we study rate of the sampling operation as we use it to estimate all values in a filtering problem.
\begin{theorem}
\label{thm:2}
Let $\mathcal{F}$ be a class of real-valued, continuous functions closed under both multiplication and marginalization, over an input set $\mathcal{X}$. Assume that $\mathcal{F}$ has a finite co-VC-dimension 
$D$. Consider a filtering problem with $T$ time steps, where the beliefs, transition and observation functions are all represented as weighted sums of $K$ functions from $\mathcal{F}$. Assume that the values of all functions $h(\mathbf{x})$ estimated in the filtering problem are upper-bounded by $M$, and that the quantity $\int h(\mathbf{x}) \hspace{0.05cm} d\mathbf{x}$ is lower-bounded by $B$. With probability at least $1 - \delta$, for all $i$ and $\mathbf{x}$, 
\begin{align*}
    (1 - \epsilon) \hspace{0.05cm} bel_i(\mathbf{x}) \leq \Tilde{bel}_i(\mathbf{x}) \leq (1 + \epsilon) \hspace{0.05cm} bel_i(\mathbf{x})
\end{align*}
if the sample size $K_{min}(\epsilon, \delta)$ used for all multiplication operations is at least
\begin{align*}
\mathcal{O}\left(
\frac{T^2}{\epsilon^2} \frac{M}{B} \left( D + D\ln{\frac{T}{\epsilon}} + D\ln{\left( \frac{M}{B}D \right)} + \ln{\frac{8}{\delta}} \right)
\right)
\end{align*}
\end{theorem}
\begin{proof}
To derive the rate of convergence, observe that all inference errors are due to approximate multiplication of two weighted sums of functions. At each step, an infinite number of values over the continuous space are estimated by our algorithm. Inequality (\ref{eq:bound_all_x}) bounds function values over all $\mathbf{x} \in \mathcal{X}$ using co-VC-dimension.

Specifically, consider two normalized weighted sums of functions containing $K_1$ and $K_2$ component functions from $\mathcal{F}$: $g_1(\mathbf{x}) = \sum_{i=1}^{K_1} w_{i} f_{i}(\mathbf{x})$ and $g_2(\mathbf{x}) = \sum_{j=1}^{K_2} v_{j} h_{j}(\mathbf{x})$, with $\sum_{i=1}^{K_1} w_{i} = \sum_{j=1}^{K_2} v_{j} = 1$ and each $f_i, h_j \in \mathcal{F}$. By closure of multiplication, the product $g(\mathbf{x}) = g_1(\mathbf{x}) g_2(\mathbf{x}) = \sum_{i=1}^{K_1} \sum_{j=1}^{K_2} w_{i} v_{j} f_{i}(\mathbf{x}) h_{j}(\mathbf{x})$ is in the convex hull of $\mathcal{F}$. The sampling operation then draws $K$ indices $\{ (k_r, l_r) \}^{K}_{r=1}$ from the distribution $w_i v_j$, forming a new weighted sum $\Tilde{g}(\mathbf{x}) = \frac{1}{K} \sum_{r=1}^{K} f_{k_r}(\mathbf{x}) h_{l_r}(\mathbf{x})$.
We can use Theorem \ref{thm:1} to bound the probability that the estimate of any $x \in X$ is outside $[(1-\zeta) g(\mathbf{x}), (1+\zeta) g(\mathbf{x})]$, giving
\begin{align}
\begin{split}
    P\left\{ \Tilde{g}(\mathbf{x}) \not\in [(1-\zeta) g(\mathbf{x}), (1+\zeta) g(\mathbf{x})] \text{ for any } x \right\} \\
    < 8 (2K)^D \exp{\left( -\frac{\zeta^2}{4} \frac{B}{M} K \right) }
\end{split}
\end{align}
where $B$ is a lower bound of $\int g(\mathbf{x}) \hspace{0.05cm} d\mathbf{x}$ and $M$ is an upper bound of $g(\mathbf{x})$. If we set an upper-bound on this error probability of $\delta$ and rearrange for $K$, we have that with probability at least $\delta$, all estimates are within a factor of $(1 \pm \delta)$ of their true values when
\begin{align}
\begin{split}
\label{eq:K_bound_1}
\hspace{-0.4cm}
    K &\geq \frac{4}{\zeta^2} \frac{M}{B} \ln{\frac{8(2K)^D}{\delta}} \\
    &= \frac{4}{\zeta^2} \frac{M}{B} \left( D \ln{K} + \left( \ln{\frac{8}{\delta}} + D \ln{2} \right) \right)
\end{split}
\end{align}
Lemma A.2 in \citet{shalev2014understanding} states that if $a \geq 1$, $b > 0$, then $x \geq 4 a \ln{(2a)} + 2b \Rightarrow x \geq a \ln{(x)} + b$. Setting $a = \frac{4}{\zeta^2} \frac{M}{B} D$ and $b = \frac{4}{\zeta^2} \frac{M}{B} (\ln{\frac{8}{\delta}} + D \ln{2})$ gives a relaxation over (\ref{eq:K_bound_1}):
\begin{align}
\begin{split}
\label{eq:K_bound_2}
    K &\geq 4a\ln{(2a)} + 2b \\
    &= \frac{8}{\zeta^2} \frac{M}{B}
    \left( 2D \ln{\left( \frac{8}{\zeta^2} \frac{M}{B} D \right)} + \ln{\frac{8}{\delta}} + D \ln{2} \right)
\end{split}
\end{align}
Next, consider the Bayes filter update over $T$ steps: 
\begin{equation}
bel_{t}(\mathbf{x}) = p(\mathbf{z} \mid \mathbf{x}) \int p(\mathbf{x} \mid \mathbf{u}, \mathbf{x'}) \hspace{0.05cm} bel_{t-1}(\mathbf{x'}) \hspace{0.05cm} d\mathbf{x'}
\end{equation}
There are two separate multiplications of weighted sums during a single time step, one inside and the other outside the integration. In total, there are $2T$ multiplications across the whole chain. We seek an expression for the sample size $K$ required for belief estimates across all time steps to have small errors.

At worst, each multiplication of weighted sums results in an extra $(1 \pm \zeta)$ factor in the bound. As we have $2T$ multiplications in total, the final belief estimates across the chain are all within a factor $(1 \pm \zeta)^{2T}$ of the true values. To bound the estimates so that all are within a factor $(1 \pm \epsilon)$ of their true values for a given $\epsilon > 0$, we note that choosing $\zeta = \frac{\ln{(1 + \epsilon)}}{2T}$ implies $(1 - \zeta)^{2T} \geq 1 - \epsilon$ and $(1 + \zeta)^{2T} \leq 1 + \epsilon$, by (9) and (10) in \citet{wrigley2017tensor}. Substituting this $\zeta$ into (\ref{eq:K_bound_2}), we have that with probability at least $1 - \delta$, all belief estimates are accurate within factor $(1 \pm \epsilon)$, when
\begin{equation*}
\begin{split}
\label{eq:K_bound_3}
    K \geq &\frac{32 T^2}{(\ln(1 + \epsilon))^2} \frac{M}{B} \cdot
    \\&
    \left( 2D \ln{\left(\frac{32 T^2}
    {(\ln{(1 + \epsilon)})^2}
    \frac{M}{B} D \right)}
    +  \ln{\frac{8}{\delta}} + D \ln{2} \right)
\end{split}
\end{equation*}
Using the facts that $\ln{(1 + \epsilon)} \geq \epsilon \cdot \ln{2}$ for $0 \leq \epsilon \leq 1$, $\frac{32}{(\ln{2})^2} < 67$ and setting $C = 2\ln{67} + \ln{2}$, we can relax this bound to
\begin{equation}
\label{eq:K_bound_final}
    K \geq \frac{67 T^2}{\epsilon^2} \frac{M}{B} \left( CD + 4D\ln{\frac{T}{\epsilon}} + 2D\ln{\left( \frac{M}{B}D \right)} + \ln{\frac{8}{\delta}} \right)
\end{equation}
\end{proof}
\section{Theoretical Analysis: Gaussians}
\label{sec:theory_gaussian}
We now study the use of Gaussian functions for our approximate inference. We denote a Gaussian function over $\mathbf{x}$ with mean $\boldsymbol{\mu}$ and variance $\mathbf{\Sigma}$ as $\mathcal{N}(\mathbf{x}; \boldsymbol{\mu}, \mathbf{\Sigma})$; and the exponential component as $\exp{(\mathbf{x}; \boldsymbol{\mu}, \mathbf{\Sigma})} = \exp \left( -\frac{1}{2} (\mathbf{x} - \boldsymbol{\mu})^{T} \mathbf{\Sigma}^{-1} (\mathbf{x} - \boldsymbol{\mu}) \right)$. Recall that $d$-dimensional Gaussian functions are given by
\begin{align*}
    \mathcal{N}(\mathbf{x}; \boldsymbol{\mu}, \mathbf{\Sigma}) &= \frac{1}{(2\pi)^{d/2} \sqrt{\det{\mathbf{\Sigma}}}} \exp{(\mathbf{x}; \boldsymbol{\mu}, \mathbf{\Sigma})}
\end{align*}
Specific properties of Gaussians make them suitable for inference in filtering problems when functions are decomposed as weighted sums of Gaussians.

\subsection{Gaussian Multiplication}
Multiplication of two Gaussian functions, $\mathcal{N}(\mathbf{x}; {\boldsymbol{\mu}_{1}}, {\mathbf{\Sigma}_{1}})$ and $\mathcal{N}(\mathbf{x}; {\boldsymbol{\mu}_{2}}, {\mathbf{\Sigma}_{2}})$, results in a third Gaussian $\mathcal{N}(\mathbf{x}; {\boldsymbol{\mu}_{3}}, {\mathbf{\Sigma}_{3}})$ with a constant factor $c$ \citep{petersen2008matrix}:
\begin{align}
\label{eq:gaussian_multiplication}
\mathcal{N}(\mathbf{x}; {\boldsymbol{\mu}_{1}}, {\mathbf{\Sigma}_{1}}) \cdot \mathcal{N}(\mathbf{x}; {\boldsymbol{\mu}_{2}}, {\mathbf{\Sigma}_{2}}) &= c \hspace{0.1cm} \mathcal{N}(\mathbf{x}; {\boldsymbol{\mu}_{3}}, {\mathbf{\Sigma}_{3}})
\end{align}
\begin{align*}
\begin{split}
    c &= (2\pi)^{-d/2} \left( \det{({\mathbf{\Sigma}_1} + {\mathbf{\Sigma}_2})} \right)^{-1/2} \cdot \\
    &\quad \exp \left( -\frac{1}{2}({\boldsymbol{\mu}_1}- {\boldsymbol{\mu}_2})^{T}({\mathbf{\Sigma}_1} + {\mathbf{\Sigma}_2})^{-1} ({\boldsymbol{\mu}_1}- {\boldsymbol{\mu}_2}) \right) \\
    {\mathbf{\Sigma}_{3}} &= \mathbf{\Sigma}_1 ({\mathbf{\Sigma}_{1}} + {\mathbf{\Sigma}_{2}})^{-1} \mathbf{\Sigma}_2\\
    {\boldsymbol{\mu}_{3}} &= \mathbf{\Sigma}_2 (\mathbf{\Sigma}_1 + \mathbf{\Sigma}_2)^{-1} \boldsymbol{\mu}_1 + \mathbf{\Sigma}_1 (\mathbf{\Sigma}_1 + \mathbf{\Sigma}_2)^{-1} \boldsymbol{\mu}_2
\end{split}
\end{align*}
While multiplying two weighted sums of Gaussians, we use the max-norm reweighting scheme in \citet{wrigley2017tensor} to make multiplications more effective. Under the scheme, in effect, we perform multiplications over $\exp \left( -\frac{1}{2} (\mathbf{x} - \boldsymbol{\mu})^{T} \mathbf{\Sigma}^{-1} (\mathbf{x} - \boldsymbol{\mu}) \right)$. Details are given in supplementary materials.

\subsection{Closed-Form Transition Update}
\label{sec:general_transition}
Special properties of Gaussians enable a closed-form transition update with an exact integration. Specifically, consider Gaussian functions over $\mathbf{x}$ and $\mathbf{x'}$: $\mathcal{N}(\mathbf{x}, \mathbf{x'}; \boldsymbol{\mu}, \mathbf{\Sigma})$. Assume that mean $\boldsymbol{\mu} \in \mathbb{R}^{2d}$ contains two parts: $\boldsymbol{\mu} = \begin{pmatrix}
    \boldsymbol{\mu}^{x} & \boldsymbol{\mu}^{x'}
\end{pmatrix}^{T}$ where $\boldsymbol{\mu}^{x}$ and $\boldsymbol{\mu}^{x'}$ are means over $\mathbf{x}$ and $\mathbf{x'}$ respectively; and that covariance 
$\mathbf{\Sigma} \in \mathbb{R}^{(2d)^2} = 
  \begin{pmatrix}
    \mathbf{\Sigma}^{xx} & \mathbf{\Sigma}^{xx'}\\
    \mathbf{\Sigma}^{xx'} & \mathbf{\Sigma}^{x'x'}
  \end{pmatrix}$.
Denote belief and transition as weighted sums of Gaussians: $bel_{t-1}(\mathbf{x}) = \sum_{i=1}^{K_1} w_i \hspace{0.05cm} \mathcal{N}(\mathbf{x}; \mathbf{a}_{i}, \mathbf{A}_{i})$; and $p(\mathbf{x} \mid \mathbf{u}, \mathbf{x'}) = \sum_{j=1}^{K_2} v_j \hspace{0.05cm} \mathcal{N}(\mathbf{x}, \mathbf{x'}; \boldsymbol{b}_{j}, \mathbf{B}_{j})$. The transition update yields a new sum of Gaussians,
\begin{align*}
\hspace{-0.4cm}
\bar{bel}_{t}(\mathbf{x}) &= \sum_{i=1}^{K_1} \sum_{j=1}^{K_2} w_i v_j \int \mathcal{N}(\mathbf{x'}; \mathbf{a}_{i}, \mathbf{A}_{i}) \hspace{0.1cm} \mathcal{N}(\mathbf{x}, \mathbf{x'}; \boldsymbol{b}_{j}, \mathbf{B}_{j}) \hspace{0.1cm} d\mathbf{x'} \\
&= \sum_{i=1}^{K_1} \sum_{j=1}^{K_2} w_i v_j z_{ij} \hspace{0.1cm} \mathcal{N}(\mathbf{x}; \boldsymbol{c}_{ij}, \mathbf{C}_{ij})
\end{align*}
where\vspace{-0.4cm}
\begin{align*}
    \mathbf{C}_{ij} &= \left( \mathbf{A}_{i}^{-1} + (\mathbf{B}_{j}^{xx} )^{-1} \right)^{-1}\\
    \boldsymbol{c}_{ij} &= \mathbf{C}_{ij} \left( \mathbf{A}_{i}^{-1} \mathbf{a}_{i} +
    (\mathbf{B}_{j}^{xx})^{-1} \boldsymbol{b}^{x}_{j}
    \right)
\end{align*}
The expression for $z_{ij}$ is complex, and is given in supplementary materials along with full derivations. The transition is simpler for the special case of robot localization, as we shall see in Section \ref{sec:localization_transition}.

\subsection{co-VC-Dimension}
The appearance of co-VC-dimension $D$ in (\ref{eq:K_bound_final}) hints at the importance of a small $D$ to guarantee a relatively fast convergence. Under the max-norm reweighting scheme, we use Gaussian exponentials, $\exp \left( -\frac{1}{2} (\mathbf{x} - \boldsymbol{\mu})^{T} \mathbf{\Sigma}^{-1} (\mathbf{x} - \boldsymbol{\mu}) \right)$, in our multiplications. These exponentials are a good candidate for our proposed technique, as they have a smallest achievable co-VC-dimension that is linear in the dimension of state space. To compute related dimensions, we use a theorem from \cite{goldberg1995bounding}:
\begin{theorem}
\label{thm:goldberg}
Let $\{ \mathcal{S}_{k, n}: k, n \in \mathbb{N} \}$ be a class of sets, where each set in $\mathcal{S}_{k, n}$ is fully specified using $k$ real values while an instance in a set is represented by $n$ real values. Suppose that the membership test for any instance $a$ in any set $S \in \{ \mathcal{S}_{k, n} \}$ can be expressed using $s$ polynomial inequality or equality predicates over $k + n$ variables of degree at most $m$. Then,  $VCdim(\mathcal{S}_{k, n}) \leq 2k \log{(2ems)}$ where $e$ is Euler's number.
\end{theorem}
Theorem \ref{thm:goldberg} offers a convenient way to compute the VC-dimension of a class of sets. It is used in the following theorem to determine the VC-dimension of a subgraph, which in turn facilitates calculation of the co-VC-dimension of Gaussian exponentials.
\begin{theorem}
\label{thm:co-vc}
The co-VC-dimension of $d$-dimensional functions of the form $\exp \left( -\frac{1}{2} (\mathbf{x} - \boldsymbol{\mu})^{T} \mathbf{\Sigma}^{-1} (\mathbf{x} - \boldsymbol{\mu}) \right)$ is $\mathcal{O}(d)$.
\end{theorem}
\begin{proof}
By definition, the co-VC-dimension of the class $\mathcal{F}$ of $d$-dimensional exponentials is the pseudo-dimension of the dual $\mathcal{F}'$. In the dual $\mathcal{F}'$ containing evaluation functions, the input space is over the mean $\boldsymbol{\mu}$ and covariance $\mathbf{\Sigma}$, while $x$ becomes a fixed parameter. The pseudo-dimension of $\mathcal{F}'$ is in turn equal to the VC-dimension of its subgraph, the class of sets $\{ (\boldsymbol{\mu}, \mathbf{\Sigma}, y) \in (\mathbb{R}^{d}, \mathbb{R}^{d^2}, \mathbb{R}) ; y \leq ev_{\mathbf{x}}(\boldsymbol{\mu}, \mathbf{\Sigma}) \}$ for $ev_{\mathbf{x}} \in \mathcal{F}'$.

Denote the class of sets as $\mathcal{S}_{k, n}$. Since a set in $\mathcal{S}_{k, n}$ is fully specified by its parameter $\mathbf{x} \in \mathbb{R}^{d}$, $k = d$ in $\mathcal{S}_{k, n}$. An instance in a set is represented by $n \in \mathcal{O}(d^2)$ values, as the input $(\boldsymbol{\mu}, \mathbf{\Sigma}, y) \in \mathbb{R}^{d^2+d+1}$. Given an instance $(\boldsymbol{\mu}, \mathbf{\Sigma}, y)$ and a set $S \in \mathcal{S}_{k, n}$, the following inequality represents membership test.
\begin{align*}
\begin{split}
    y \leq ev_{\mathbf{x}}(\boldsymbol{\mu}, \mathbf{\Sigma})
    &= \exp \left( -\frac{1}{2} (\mathbf{x} - \boldsymbol{\mu})^{T} \mathbf{\Sigma}^{-1} (\mathbf{x} - \boldsymbol{\mu}) \right)\\
    \Rightarrow 2\log{y} &\leq - (\mathbf{x} - \boldsymbol{\mu})^{T} \mathbf{\Sigma}^{-1} (\mathbf{x} - \boldsymbol{\mu})
\end{split}
\end{align*}
The above test can be transformed into two predicates: one inequality and one equality. The inequality is used for the value comparison; while the right-hand side, $(\mathbf{x} - \boldsymbol{\mu})^{T} \mathbf{\Sigma}^{-1} (\mathbf{x} - \boldsymbol{\mu})$, can be represented using an equality predicate with degree $m = 3$.

A direct application of Theorem \ref{thm:goldberg} gives an upper bound for the VC-dimension of the subgraph as $2k \log{(8 ems) = 2d\log{(24e)}} \in \mathcal{O}(d)$. This is also the co-VC-dimension of $d$-dimensional Gaussian exponentials.
\end{proof}

\section{Application to Robot Localization}
To demonstrate potential of the proposed approach, we test it on a classic filtering problem -- robot localization. A robot is placed in a previously unseen environment and does not know its location or orientation. It navigates for a number of time steps, relying on an environment map, a number of onboard sensors and its movement data to gradually infer its pose. At each step, inputs to the robot are RGBD image observations, $\mathbf{o}_t$; odometry, $\mathbf{u}_t$; and the map, $\mathbf{M}$. The robot uses these data to infer its pose $(x, y, \theta)$ over time.

In this section, we consider details of the Bayes filter in the context of robot localization. As a proof-of-concept, we use a handcrafted model in a semi-realistic environment -- realistic buildings, but without furniture. For transition, we describe a simple update based on odometry inputs. For observation, we design a handcrafted model using only depth information to generate a number of Gaussians to represent $p(\mathbf{z} \mid \mathbf{x})$. In more complex and realistic settings, these functions can be learned from data instead of being handcrafted. Experimental design and results are discussed in the next section.

\subsection{Initialization}
All filtering problems start with an initial belief $bel_0(\mathbf{x})$. In the context of robot localization, $bel_0(\mathbf{x})$ is initialized as a set of Gaussians, with means over the 3D space of poses $(x, y, \theta)$ and covariances manually chosen. We consider $3$ settings for initial Gaussian centers, with increasing uncertainty. For \textit{tracking}, the robot is initially already well localized and we would like to track the robot well as it moves. Initial beliefs are distributed around the true state: the centers follow a Gaussian distribution around the true pose, with zero mean and covariance matrix $\mathbf{\Sigma} = \text{diag}[30\text{cm}, 30\text{cm}, 30^\circ]$. For \textit{semi-global localization}, the particles are uniformly initialized in the same room as the true state. We also consider a two-room case where an additional room is randomly chosen to initialize Gaussians. For \textit{global localization}, positions $(x, y)$ are randomly picked from empty map positions over the whole house. Particle orientations $\theta$ for both semi-global and global localization are uniformly sampled from the interval $[-\pi, \pi)$. 

For simplicity, covariances in all tasks are represented using diagonal matrices, with an independence assumption among $x, y, \theta$ components. For tracking, initial covariances are $\text{diag}[(4\text{cm})^2, (4\text{cm})^{2}, (0.1\text{rad})^{2}]$.  For semi-global and global localization, initial covariances are $\text{diag}[(40\text{cm})^2, (40\text{cm})^{2}, (1\text{rad})^{2}]$ and $\text{diag}[(200\text{cm})^2, (200\text{cm})^{2}, (1\text{rad})^{2}]$, respectively. Large covariances are exploited in tasks with higher uncertainty to provide more space coverage.  The covariance values have been manually checked to give good performance for the respective tasks. Gaussian weights are initialized to be uniform.

\begin{figure*}[ht]
\centering  
\includegraphics[width=0.95\textwidth]{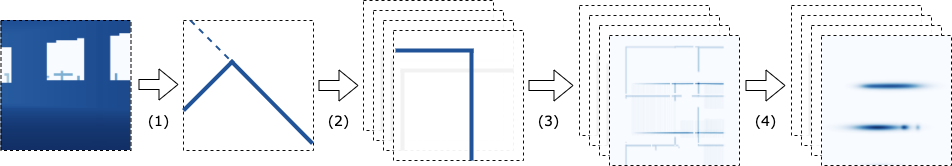}
\caption{Pipeline for extracting Gaussians using depth image. (1): simulated laser scan, followed by edge detection; (2): rotation; (3): convolution over wall map; (4): Gaussian extraction.}
\label{fig:pipeline}
\end{figure*}

\subsection{Transition Model}
\label{sec:localization_transition}
Given pose $\mathbf{x'}$ and odometry $\mathbf{u} = (\Delta x, \Delta y, \Delta \theta)$ representing robot's relative motion at the current time step, the transition function $p(\mathbf{x} \mid \mathbf{u}, \mathbf{x'})$ designates probabilities to possible next poses. While we considered representing $p(\mathbf{x} \mid \mathbf{u}, \mathbf{x'})$ as a weighted sum of Gaussians, using one Gaussian suffices in the special case of localization. If the previous pose is $\mathbf{x'}$ and the robot undergoes a displacement described by $\mathbf{u}$ with Gaussian noise $\mathcal{N}(\mathbf{x}; \mathbf{0}, {\mathbf{\Sigma}_{c}})$ considered, intuitively, the distribution
\begin{equation*}
    p(\mathbf{x} \mid \mathbf{u}, \mathbf{x'}) = \mathcal{N}(\mathbf{x}; \mathbf{x'} + \mathbf{u}, {\mathbf{\Sigma}_{c}})
\end{equation*}
Assume $bel_{t-1}(\mathbf{x}) = \sum_{i=1}^{K} w_i \hspace{0.05cm} \mathcal{N}(x; \boldsymbol{\mu}_{i}, \mathbf{\Sigma}_{i})$; then, the transition update 
\begin{align*}
    \bar{bel}_{t}(\mathbf{x}) &= \int p(\mathbf{x} \mid \mathbf{u}, \mathbf{x'}) \hspace{0.05cm} bel_{t-1}(\mathbf{x'}) d\mathbf{x'}\\
    &= \sum_{i=1}^{K} w_i \int \mathcal{N}(\mathbf{x}; \mathbf{x'} + \mathbf{u}, {\mathbf{\Sigma}_{c}}) \hspace{0.05cm} \mathcal{N}(\mathbf{x'}; \boldsymbol{\mu}_{i}, \mathbf{\Sigma}_{i}) \hspace{0.05cm} d\mathbf{x'}\\
    &= \sum_{i=1}^{K} w_i \int \mathcal{N}(\mathbf{x'}; \boldsymbol{\mu}_{i}, \mathbf{\Sigma}_{i}) \hspace{0.05cm} \mathcal{N}(\mathbf{x} - \mathbf{x'}; \mathbf{u}, {\mathbf{\Sigma}_{c}}) \hspace{0.05cm} d\mathbf{x}'\\
    &= \sum_{i=1}^{K} w_i \hspace{0.05cm} \mathcal{N}(\mathbf{x}; \boldsymbol{\mu}_{i}, \mathbf{\Sigma}_{i}) * \mathcal{N}(\mathbf{x}; \mathbf{u}, {\mathbf{\Sigma}_{c}}) \hspace{0.05cm} \\
    &= \sum_{i=1}^{K} w_i \hspace{0.05cm} \mathcal{N}(\mathbf{x}; \boldsymbol{\mu}_{i} + \mathbf{u}, \mathbf{\Sigma}_{i} + {\mathbf{\Sigma}_{c}})
\end{align*}
where at the second last equality, we apply the formula for the convolution of two Gaussian functions:
\begin{align*}
\mathcal{N}(\mathbf{z}; \hspace{0.05cm} &{\boldsymbol{\mu}_{1}}, {\mathbf{\Sigma}_{1}}) * \mathcal{N}(\mathbf{z}; {\boldsymbol{\mu}_{2}}, {\mathbf{\Sigma}_{2}}) \\ 
&=
\int \mathcal{N}(\mathbf{x}; {\boldsymbol{\mu}_{1}}, {\mathbf{\Sigma}_{1}}) \hspace{0.05cm} \mathcal{N}(\mathbf{z} - \mathbf{x}; {\boldsymbol{\mu}_{2}}, {\mathbf{\Sigma}_{2}}) d{\mathbf{x}} \\
&= \mathcal{N}(\mathbf{z}; {\boldsymbol{\mu}_{1}} + {\boldsymbol{\mu}_{2}}, {\mathbf{\Sigma}_{1}} + {\mathbf{\Sigma}_{2}}) 
\end{align*}

\subsection{Observation Model}
Given observation $\mathbf{z}$, the function $p(\mathbf{z} \mid \mathbf{x})$ over $\mathbf{x}$, is approximated similarly as a weighted sum of Gaussians: $p(\mathbf{z} \mid \mathbf{x}) = \sum_{i=1}^{K} w_{i} \hspace{0.05cm} \mathcal{N}(\mathbf{x}; \boldsymbol{\mu}_{i}, \mathbf{\Sigma}_{i})$, $\boldsymbol{\mu} \in \mathbb{R}^{3}, \mathbf{\Sigma} \in \mathbb{R}^{3 \times 3}$, where $K$ represents the number of Gaussians identified from observations. 

How do we obtain the set of Gaussians at each step? We propose a pipeline for this purpose, illustrated in Figure \ref{fig:pipeline}. Given a depth image as observation $\mathbf{z}$, our method transforms it into a filter representing shape of the wall facing the robot. Next, the filter is used to perform convolution over the wall map, generating probability distributions over map locations. Gaussians are then extracted from these distributions. We follow the Manhattan world aussumption \citep{coughlan2001manhattan} to simplify inference, and assume that wall intersection angles are $90^{\circ}$. Our data set conforms to this assumption.

\subsubsection{Filter Generation}  The first step is to generate a filter representing shape of the wall, typically corners, currently facing the robot. Given a depth image containing distances from the camera, we perform a simulated laser scan horizontally across the depth image. The $80$-percentile value of each column is taken and multiplied with a pre-defined constant to represent the distance.
The simulated LIDAR has a resolution of 56 beams and a $60^\circ$ field of view. The resulting filter typically has one or two main edges. RANSAC regression is then applied to identify an edge from the filter.

\subsubsection{Rotation and Convolution} Under the Manhattan world assumption, there are $4$ directions along which to rotate the filter so that the end result aligns with the map. Afterwards, we apply convolution on the wall map, using the rotated filters as the convolution kernel. This gives $4$ distributions over the map; values on these distributions represent how well map locations match with the filter.

\subsubsection{Gaussian Extraction}
Finally, we perform thresholding and segmentation on belief images to extract the most important regions. We take $50\%$ of the maximum match value from each distribution for thresholding. Outcomes are typically sets of line segments. For a line segment, a local maximum is taken as the Gaussian center while the maximum on this segment is taken as the weight. For covariance along $(x, y)$ directions, we take the maximum distance from the center as the standard deviation, and manually increment it with a handpicked value ($40\text{cm}$) to increase coverage. For the $\theta$ direction, we pick 
$\pi \text{rad}$ as the standard deviation.

\section{Simulation Experiments}
We conduct localization experiments on the House3D simulator \citep{wu2018building}, built on top of a collection of residential buildings from the SUNCG data set \citep{song2017semantic}. We consider a simplified environment without furniture. The average building and room sizes are $206 \text{m}^2$ and $37 \text{m}^2$, respectively. We use TensorFlow for our implementation\footnote{Our implementation is available at \url{https://github.com/suxuann/mgpf}.}.

\subsection{Sampling in Practice}
While using sampling to approximate multiplication gives an unbiased estimate, in practice it has high variance especially if we sample from $K^2$ terms when $K$ is a large number, giving low performance in our experiments. An alternative method to reduce the number of terms in a product is to take the largest coefficients. Given a weighted sum of functions containing $K^2$ terms  $\sum_{i=1}^{K^2} w_i f_{i}(\mathbf{x})$, the terms are sorted in decreasing order of $w_i$ to become $\sum_{j=1}^{K^2} \tilde{w}_j \tilde{f}_{j}(\mathbf{x})$ with $\tilde{w}_1 \geq \tilde{w}_2 \geq \dots \geq \tilde{w}_{K^2}$. The largest $K$ weights $\tilde{w}_j, 1 \leq j \leq K$ are taken to form a new weighted sum $\sum_{j=1}^{K} \tilde{w}_j \tilde{f}_{j}(\mathbf{x})$. We call this alternative sampling method \emph{top-$K$ sampling}. While top-$K$ sampling gives a biased estimate, it is found to be more effective in our localization experiments.

\subsection{Baseline}
We experimentally compare our inference method -- multiplicative Gaussian particle filter (MGPF) -- with the particle filter (PF). For a fair comparison, we constrain them to use the same transition and observation models as described above. The only difference lies in how inference is performed. At each step, MGPF multiplies the set of current Gaussian beliefs, with another set representing observation $P(\mathbf{z}\mid\mathbf{x})$; while PF reads off weights of particles directly from $P(\mathbf{z}\mid\mathbf{x})$, followed by an optional resampling operation. For all our experiments, the resampling step is activated.

We hypothesize that inclusion of covariance in MGPF increases the capacity of particle coverage. Indeed, experimental results show that while the two filtering methods have comparable performance under concentrated initial beliefs (tracking), MGPF consistently outperforms PF in alternative cases with higher uncertainty (semi-global and global localization).

\subsection{Evaluation} We randomly generate a localization data set in the House3D simulator for evaluation. At each time step, the robot takes one of two actions: moves forward ($p = 0.8$), or makes a turn ($p = 0.2$). The moving distance and turning angle are uniformly taken from $[20\text{cm}, 80\text{cm}]$ and $[15^\circ, 60^\circ]$, respectively. Each trajectory contains $100$ time steps. The final evaluation set consists of $820$ trajectories in $47$ different buildings. Our evaluation settings largely follow those of \cite{karkus2018particle}.

\begin{figure}[ht]
\centering  
\includegraphics[width=0.40\textwidth]{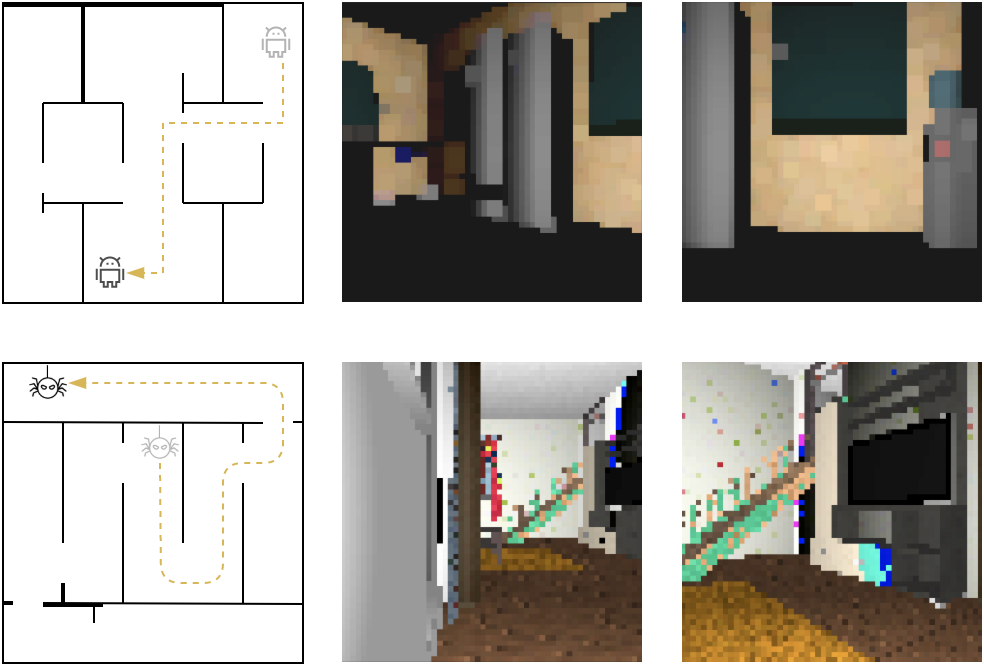}
\caption{Example House3D localization data.}
\label{fig:data}
\end{figure}

For tracking, we take the first $24$ steps of each trajectory. We report the mean average error (MAE) and the root mean square error (RMSE) calculated for robot positions along each trajectory. For semi-global and global localization, we consider full trajectories. We use another evaluation metric -- success rate: localization is considered successful if the estimation error falls below $1\text{m}$ for the last $25$ steps of a $100$-step trajectory. We vary the numbers of particles $K$ for each task. As uncertainty increases, more particles are needed for sufficient space coverage.

\subsection{Results and Discussion}
Tables \ref{tab:tracking} and \ref{tab:localization} report results of tracking, semi-global and global localization. It is clear experimentally that expanding discrete states (i.e. the Gaussian centers) with covariances as done in MGPF enlarges particle coverage. At each step, the inference method in MGPF redistributes and reweighs current Gaussians according to true Gaussians identified from observations, and successfully reduces state uncertainty: it reaches a success rate of $71.10\%$ for one-room localization, and $56.10\%$ for localization over the whole house. These experimental results present an established proof-of-concept of our proposed inference method.

\begin{table}[h!]
\centering
\begin{tabular}{||c c c c||} 
\hline
 & $K$ & MAE & RMSE\\ [0.5ex] 
\hline\hline
PF & 50 & 46.81 & 69.96\\
PF & 100 & 35.43 & 49.26\\ 
PF & 300 & \textbf{25.02} & \textbf{35.31}\\
MGPF & 50 & 29.02 & 43.94\\
MGPF & 100 & 27.29 & 42.28\\
MGPF & 300 & 25.33 & 40.01\\
\hline
\end{tabular}
\caption{Tracking results in RMSE (cm) and MAE (cm), using $K$ particles or Gaussian functions.}
\label{tab:tracking}
\end{table}
\vspace{-0.1cm}
\begin{table}[h!]
\centering
\begin{tabular}{||c c c c c||} 
\hline
 & $K$ & $N = 1$ & $N = 2$ & $N =$ all\\ [0.5ex] 
\hline\hline
PF & 100 & 4.27 & 2.56 & 0.98\\
PF & 300 & 8.78 & 6.34 & 2.44\\
PF & 600 & 14.27 & 7.68 & 2.80\\
MGPF & 100 & 58.82 & 44.51 & 44.39\\
MGPF & 300 & 69.27 & 63.78 & 54.02\\
MGPF & 600 & \textbf{71.10} & \textbf{68.17} & \textbf{56.10}\\
\hline
\end{tabular}
\caption{Localization over $N$ rooms. Results in success rate  (\%), using $K$ particles or Gaussian functions.}
\label{tab:localization}
\end{table}

In contrast, while the conventional PF is able to maintain the localized states in tracking, it fails to decrease state uncertainties in localization tasks, reaching a success rate of barely $14.27\%$ for localization. In addition, we note that as expected, increasing the number of particles and Gaussian functions helps for both models. Using more particles or functions to cover the state space leads to gains in success rate.

\section{Conclusions and Future Work}
In this paper, we propose a new parameterization of the Bayes filter, based on decomposing functions into weighted sums of continuous functions. Theoretically, we analyze our approximation to give its convergence rate, showing its relationship with the co-VC-dimension of the given function class. Next, we study the class of Gaussian functions, and show that it is suitable as the component function in weighted sums. We evaluate our method using robot localization as a proof-of-concept and experimentally demonstrate that the method outperforms particle filters.

The technique of using functions in place of discrete particles is general, and can be applied to additional domains such as robotic mapping or econometrics. Furthermore, learning can be used to construct models when domain knowledge is insufficient.

\subsubsection*{Acknowledgements}
We thank the anonymous reviewers for their helpful comments. This work is supported by NUS AcRF Tier 1 grant R-252-000-639-114.

\bibliography{aistats2020.bib}

\newpage
\clearpage

\end{document}